\documentclass[letterpaper, 10 pt, conference]{ieeeconf}  

\usepackage{graphics, graphicx}
\usepackage{amsfonts}
\usepackage{amsmath}
\usepackage{booktabs}
\usepackage{balance}
\usepackage[caption=false]{subfig}
\usepackage{cite}
\usepackage{algorithm,algorithmic}

\usepackage{tikz}

\newcommand\copyrighttext{%
  \footnotesize This version has been accepted for publication in Proc. of the 2020 European Control Conference (ECC). Personal use of this material is permitted. Permission
from EUCA must be obtained for all other uses, in any current or future media, including reprinting/republishing this material for advertising or promotional
purposes, creating new collective works, for resale or redistribution to servers or lists, or reuse of any copyrighted component of this work in other works.}
\newcommand\copyrightnotice{%
\begin{tikzpicture}[remember picture,overlay]
\node[anchor=south,yshift=10pt] at (current page.south) {\fbox{\parbox{\dimexpr\textwidth-\fboxsep-\fboxrule\relax}{\copyrighttext}}};
\end{tikzpicture}%
}

\IEEEoverridecommandlockouts                              

\title{\LARGE \bf
Simultaneous Perturbation Stochastic Approximation for\\ Few-Shot Learning
}

\author{Andrei Boiarov, Oleg Granichin, Olga Granichina
\thanks{A. Boiarov and O. Granichin are with the Saint Petersburg State University (Faculty of Mathematics and Mechanics and Research Laboratory for Analysis and Modeling of Social Processes), 7-9, Universitetskaya Nab., St. Petersburg, 199034, Russia. O. Granichin is also with and the Institute of Problems in Mechanical Engineering, Russian Academy of Sciences. O. Granichina is with Institute of Childhood, Herzen State Pedagogical University, St. Petersburg, Russia.
 E-mail: {\tt\small a.boiarov@spbu.ru}, {\tt\small o.granichin@spbu.ru}, {\tt\small olga.granichina@mail.ru}.}%
}

\newtheorem{theorem}{Theorem}

\begin{document}

\maketitle
\thispagestyle{empty}
\pagestyle{empty}

\copyrightnotice

\begin{abstract}

Few-shot learning is an important research field of machine learning in which a classifier must be trained in such a way that it can adapt to new classes which are not included in the training set. However, only small amounts of examples of each class are available for training. This is one of the key problems with learning algorithms of this type which leads to the significant uncertainty. We attack this problem via randomized stochastic approximation. In this paper, we suggest to consider the new multi-task loss function and propose the SPSA-like few-shot learning approach based on the prototypical networks method. We provide a theoretical justification and an analysis of experiments for this approach. The results of experiments on the benchmark dataset demonstrate that the proposed method is superior to the original prototypical networks.

\end{abstract}

\section{INTRODUCTION}

Successful operation of many standard machine learning algorithms for supervised learning requires a clear data model, the ability to calculate the gradient for the optimized loss function (quality functional) and a large number of training data that are close to normally distributed~\cite{polyak1987introduction}. However, under real world conditions, these requirements are often not fulfilled: the hypothesis of data centering is not confirmed, and it is impossible to calculate the gradient for the loss function. Therefore, standard universal methods receive a conservative estimate of the desired parameters. Thus, for such cases, it is necessary to develop new methods that can be used under non-standard conditions.

One example of such non-standard conditions is associated with the processing of weakly labeled data (in contrast to standard supervised learning pipelines~\cite{boiarov2019large, boiarov2017arabic}) and arises in the few-shot learning problem that is included in a wider range of meta-learning problems~\cite{finn2017model}. The few-shot learning algorithm should classify a whole dataset with high quality by few examples per class and adapt to new classes not seen during training. One of the promising ideas for improving the quality of such algorithms is the more careful using of the information in the loss function.

Under conditions of substantially noisy observational data, the quality of standard gradient optimization algorithms decreases. Stochastic approximation algorithms with input randomization remain operational in many cases. Therefore, for training few-shot machine learning methods in such conditions, it makes sense to use recurrent adaptive data processing algorithms, among which one often uses approaches based on stochastic approximation (SA).

In this paper we introduce and mathematically prove the SPSA-like few-shot learning approach based on prototypical networks~\cite{snell2017prototypical}. A key new feature of our contribution is a new multi-task loss function. The impact of each task in the considered loss function is optimized via SA. In addition, we show that the proposed method is superior to the original prototypical networks on the benchmark dataset under both difficult and standard conditions.

The paper is organized as follows: Section~\ref{sec:related} provides an overview of the main works related to the topic of this paper. In Section~\ref{sec:problem} we formulate the few-shot learning problem and describe the prototypical networks algorithm. Section~\ref{sec:spsa_fsl} presents our SPSA-like approach for few-shot learning and its mathematical analysis. In Section~\ref{sec:experiments} we provide results of the experiments with our method on the Omniglot dataset~\cite{lake2015human, lake2019omniglot}. Section~\ref{sec:conclusion} concludes the paper.

\section{RELATED WORKS}\label{sec:related}

The SA algorithm was first proposed by Robbins and Monro \cite{robbins1951stochastic} and was developed to solving the optimization problem by Kiefer and Wolfowitz (KW) \cite{kiefer1952stochastic} based on finite difference approximations. Spall \cite{spall1992multivariate} introduced the simultaneous perturbation stochastic approximation (SPSA) algorithm with only two observations at each iteration which recursively generates estimates along random directions. For large dimension $d$ SPSA algorithm has the same order of convergence rate as KW-procedure. Granichin~\cite{granichin1989stochastic, granichin2002randomized} and Polyak and Tsybakov~\cite{polyak1990optimal} proposed similar stochastic approximation algorithms with input randomization that use only one (or two) value of the function under consideration at a point (or points) on a line passing through the previous estimate in a randomly chosen direction. When unknown but bounded disturbance corrupts the observed data, the quality of classical methods based on the stochastic gradient decreases. However, the quality of SPSA-like algorithms remains high \cite{granichin2015randomized}. Stochastic approximation algorithms are successfully used in machine learning, more precisely for solving clustering problems~\cite{boiarov2017simultaneous, boiarov2019stochastic}.


Few-shot learning approaches can be divided into two main groups: metric based and optimization based. The idea of metric based algorithms is to compare the query example that you are trying to classify with the example that you have. This comparison can be trained via Siamese network~\cite{koch2015siamese}, learned metric space~\cite{vinyals2016matching} or prototypical networks~\cite{snell2017prototypical}. The family of optimization based approaches includes methods from~\cite{finn2017model, rusu2018meta, jamal2019task} that learn such initial representation of a deep neural network that can be effectively fine-tuned from a small number of examples. A separate class of few-shot learning algorithms includes methods that use a recurrent neural networks~\cite{santoro2016meta, ravi2016optimization}.

Multi-task learning aims to improve prediction accuracy of one model for each task compared to training a separate model for each task~\cite{ruder2017overview, kendall2018multi}. One of the most important problems of multi-task learning is tuning weights for each task in a loss function. Authors of~\cite{kendall2018multi} solve this problem by deriving a multi-task loss function based on maximizing the Gaussian likelihood with task-dependant uncertainty. 

\section{PROBLEM STATEMENT}\label{sec:problem}

According to the few-shot learning problem formulation we need to train a classifier that can adapt to the recognition of new classes, which are not saw in training, when only a few examples of each of these classes are given. Fig.~\ref{fig:tengwar_1_shot_20_way} presents the example from the Omniglot dataset~\cite{lake2015human}: handwriting characters from one alphabet. Each of the 20 characters at the bottom represents a single class, and the task is to determine to which of these classes does the top one character belong.

\begin{figure}[thpb]
\centering
\includegraphics[scale=0.32]{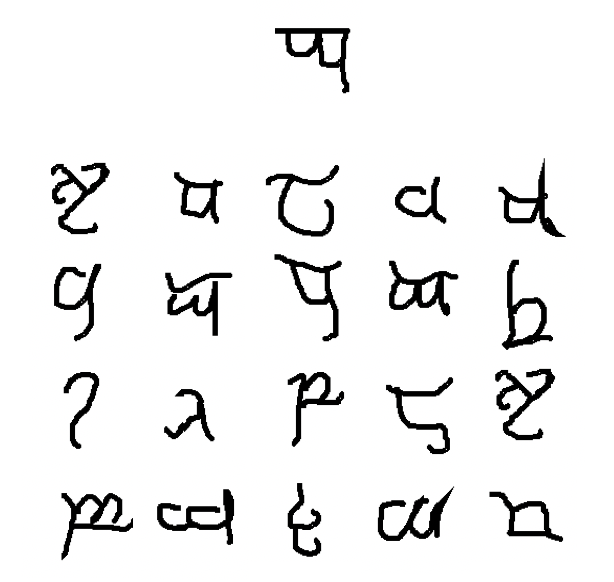}
\caption{Omniglot dataset: 1-shot 20-way classification.}
\label{fig:tengwar_1_shot_20_way}
\end{figure}

Meta-learning pipeline (few-shot learning pipeline in our case) was proposed in~\cite{vinyals2016matching} 
to train the model, capable of solving such a problem. In this pipeline, elements of each training class are divided into {\it support set} and {\it query set}. The support set consists of labeled examples, which are used to predict classes for the unlabeled examples from the query set. Another important feature of the meta-learning pipeline is the method of sampling data for training and testing. Training and testing processes consist of episodes. Each episode $\xi_t$ includes tasks, and each task $t_i$ consists of support and query sets for several classes. Classes in train tasks and test tasks do not overlap. Model training takes place on training episodes, and evaluating on test episodes. This meta-learning (few-shot learning) pipeline is shown in Fig.~\ref{fig:meta_learning_pipeline}.

\begin{figure*}[thpb]
\centering
\includegraphics[width=0.85\textwidth]{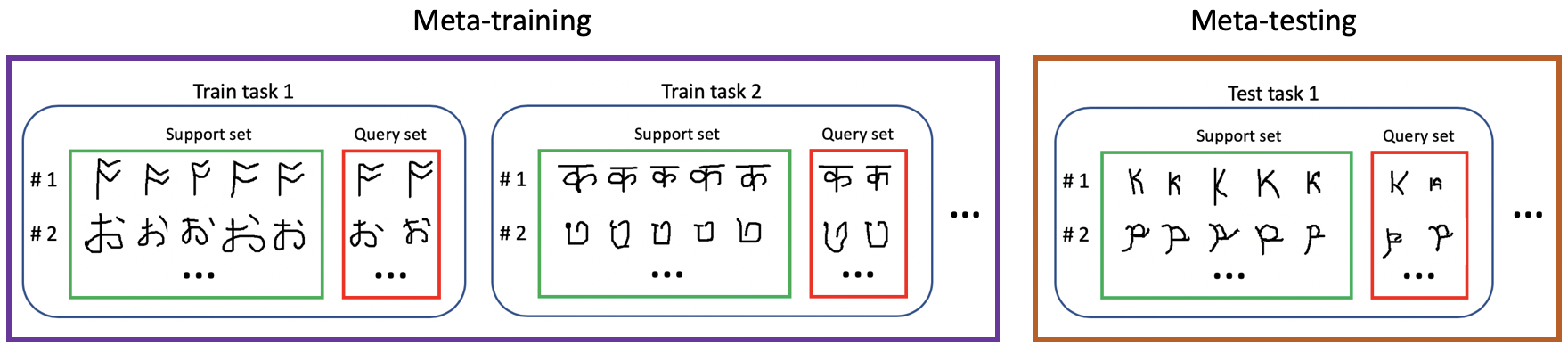}
\caption{Meta-learning (few-shot learning) pipeline with $N_S=5, N_Q=2$.}
\label{fig:meta_learning_pipeline}
\end{figure*}

Let we have $C$ classes and $N$ examples for each class in the set of labeled examples $\left\lbrace (\mathbf{x}_1, y_1),\ldots, (\mathbf{x}_{CN}, y_{CN})  \right\rbrace$ where $\mathbf{x}_i \in \mathbb{R}^d$ is the vector of an example and $y_i \in \left\lbrace 1,\ldots, C \right\rbrace$ is the class label. Let $N_S$ be the number of examples in the support set for each class, $N_Q$ be the number of examples in the query set, $N_S + N_Q = N$; $N_C \leq C$ be the number of classes in a task. For this case the few-shot learning procedure called {\it $N_S$-shot $N_C$-way}. Fig.~\ref{fig:tengwar_1_shot_20_way} represents the example of 1-shot 20-way classification problem. 

Let episode $\xi_t: (t_1, \ldots, t_M)$ consists of $M$ tasks. Each task $t_i$ contains support set $S_{t_i}$ and query set $Q_{t_i}$: ($S_{t_i}, Q_{t_i}$), where $$S_{t_i} = \left\lbrace S_{t_i}^k \right\rbrace_{k=1}^{N_C}, Q_{t_i} = \left\lbrace Q_{t_i}^k \right\rbrace_{k=1}^{N_C}, S_{t_i}^k \cap Q_{t_i}^k = \emptyset.$$ Let sets $$S_{t_i}^k = \left\lbrace x_j | y_j = k \right\rbrace_{j=1}^{N_S} \text{ and } Q_{t_i}^k = \left\lbrace x_j | y_j = k \right\rbrace_{j=1}^{N_Q}$$ be randomly selected for each task from examples of class~$k$. In standard few-shot learning approaches from~\cite{koch2015siamese, santoro2016meta, ravi2016optimization,vinyals2016matching, snell2017prototypical, finn2017model} we have $M=1$. Classes $\left\lbrace k \right\rbrace_1^{N_C}$ in each task are formed by randomly selecting a subset of classes from the training set.

\subsection{Prototypical Networks for Few-shot Learning}

Few-shot learning algorithms can be divided into two groups: optimization based and metric based approaches. One of the most popular methods is the prototypical networks algorithm~\cite{snell2017prototypical} which is a representative of the metric based family. We consider this approach because it is quite effective and can be easily generalized to different types of few-shot learning problems. 

Prototypical networks algorithm addresses the key issue of overfitting during few-shot learning. Like many modern approaches it is based on deep neural networks, through which input embedded into some numerical vector. The main idea is to train such single embedding (prototype) for each class that represents a class and points cluster around this prototype. Classification
is then performed for an embedded query point by simply finding the nearest class prototype.

Let $\phi_\theta(\mathbf{x}): \mathbb{R}^d \to \mathbb{R}^n$ be a convolutional neural network (CNN) with parameters $\theta$. In the prototypical networks method for each class $k$ computes representation $\mathbf{c}^k_{t_i} \in \mathbb{R}^n$ called {\it prototype}. Each prototype is the mean vector of the corresponding support set:
\begin{equation}\label{prototype}
\mathbf{c}^k_{t_i} = \frac{1}{|S_{t_i}^k|} \sum_{\mathbf{x}_j \in S_{t_i}^k} \phi_\theta (\mathbf{x}_j).
\end{equation}

The loss (quality) function for the class $k$ if defined as the negative log-probability that the query example $\mathbf{x}$ is belongs to the class $k$:
\begin{equation}\label{nll}
l_{\theta, t_i}^k (\mathbf{x}) = -\log \frac{\exp(-d(\phi_\theta (\mathbf{x}), \mathbf{c}^k_{t_i}))}{\sum_{k'} \exp(-d(\phi_\theta (\mathbf{x}), \mathbf{c}^{k'}_{t_i}))},
\end{equation}
where $d(\cdot, \cdot)$ is some distance function. We will futher consider the Euclidean distance.

The prototypical networks model is training via stochastic gradient descent (SGD) by minimizing loss function for train task $t_i$
\begin{equation}\label{fsl_loss}
\mathcal{L}_{\theta, t_i} (Q_{t_i}) = \frac{1}{N_C} \sum_{k=1}^{N_C} \frac{1}{N_Q} \sum_{\mathbf{x}_j \in Q_{t_i}^k} l_{\theta, t_i}^k (\mathbf{x}_j).
\end{equation}

In the original prototypical networks algorithm the number of tasks per episode $M=1$, therefore each training episode $\xi_t$ consists of one task $t_1$. The following algorithm presents the procedure of updating parameters $\theta$ of the convolutional neural network for one training episode.

\begin{algorithm}\label{alg:protonet}
\begin{algorithmic}[1]
\caption{Training for episode $\xi_t: (t_1)$}
\renewcommand{\algorithmicrequire}{\textbf{Input:}}
\renewcommand{\algorithmicensure}{\textbf{Output:}}
\REQUIRE $N_S$, $N_Q$, $N_C$
\ENSURE Updated parameters $\theta$
\STATE Random sample $N_C$ classes
\FOR{$k$ in $\left\lbrace 1,\ldots,N_C \right\rbrace$}
\STATE Random sample elements in $S_{t_1}^k$ 
\STATE Random sample elements in $Q_{t_1}^k$ 
\STATE Compute $\mathbf{c}^k_{t_1}$ via~(\ref{prototype}) 
\ENDFOR
\STATE $\mathcal{L}_{\theta, t_1}=0$
\FOR{$k$ in $\left\lbrace 1,\ldots,N_C \right\rbrace$}
\FOR{($\mathbf{x},y)$ in $Q_{t_1}^k$ }
\STATE $\mathcal{L}_{\theta, t_1}=\mathcal{L}_{\theta, t_1} + \frac{1}{N_C N_Q} l_{\theta, t_1}^k (\mathbf{x})$
\ENDFOR
\ENDFOR
\STATE Update parameters $\theta$ via SGD by $\mathcal{L}_{\theta, t_1}$
\end{algorithmic}
\end{algorithm}

\section{SPSA FOR FEW-SHOT LEARNING}\label{sec:spsa_fsl}

Prototypical networks Algorithm~1 as well as other main few-shot learning methods~\cite{koch2015siamese, santoro2016meta, ravi2016optimization,vinyals2016matching, snell2017prototypical, finn2017model} at each training episode use only one task. However, the number of tasks can be limited only by computing capabilities and time. So each episode $\xi_t$ of the few-shot learning pipeline may consist of several tasks $t_1,\ldots,t_M$. On the other hand, multi-task machine learning is a rapidly developing area in recent years and shows many successful results, especially in deep neural networks~\cite{ruder2017overview}. Therefore we will build our modification of the prototypical networks method on the new idea of using multiple tasks simultaneously per training episode.

\subsection{Multi-Task Learning}

There are two main multi-task learning approaches for deep neural networks: soft and hard parameter sharing of hidden layers of neural network~\cite{ruder2017overview}. In our method we use hard parameter sharing for all hidden layers of our convolutional network. This means that we have the same network for all tasks, and the presence of several tasks is reflected only in the loss function. For this purpose we adapted the approach proposed in~\cite{kendall2018multi} which uses task-depended (homoscedastic) uncertainty as a basis for weighting losses in a multi-task
learning problem. In~\cite{kendall2018multi} authors combine multiple regression and classification loss functions for tasks of a pixel-wise classification, an instance semantic segmentation and an estimate of per pixel depth. In the few-shot learning training pipeline tasks are more similar and loss functions have the same structure. Thus our new multi-task few-shot learning loss function with~(\ref{fsl_loss}) has the following form:

\begin{equation}\label{multitask_loss}
f_{\xi_t}(\boldsymbol{\omega}_t, \mathbf{x}) = \sum_{i=1}^M \frac{1}{(\omega_t^i)^2} \mathcal{L}_{\theta, t_i} (Q_{t_i}) + \sum_{i=1}^M \log (\omega_t^i)^2,
\end{equation}
where weights $\boldsymbol{\omega}_t = (\omega_t^1,\ldots,\omega_t^M)$ are hyper-parameters. Tuning of $\boldsymbol{\omega}_t$ is critical to success of multi-task learning. We also consider $M$ as a parameter of our algorithm. 

\subsection{SPSA}

In the proposed approach, deep convolutional neural network $\phi_\theta(\mathbf{x})$ parameters $\theta$ will be modified via SGD as in the prototypical networks algorithm. Instead we focus our attention on the multi-task parameters $\boldsymbol{\omega}_t$ in the loss function~(\ref{multitask_loss}) due to the fact that their optimization plays a key role in the whole learning algorithm. To find these parameters, we formulate the nonstationary optimization problem according to~\cite{vakhitov2009algorithm, granichin2014simultaneous}.

Consider the observation model for the training episode $\xi_t$
\begin{equation*}
L_t(\boldsymbol{\omega}_t) = f_{\xi_t}(\boldsymbol{\omega}_t, \mathbf{x}) + \nu_t,
\end{equation*}
where $\nu_t$ is an additive external noise caused by uncertainties in the calculation of the loss function~(\ref{fsl_loss}) by few examples.

Let $\mathcal{F}_t$ be the $\sigma$-algebra of all probabilistic events which happened up to time instant $t=1,2,\ldots$. Hereinafter $\mathbb{E}_{\mathcal{F}_{t-1}}$ is a symbol of the conditional mathematical expectation with respect to the $\sigma$-algebra $\mathcal{F}_{t-1}$.

Thus, the optimization problem is formulated as an estimation of the point of minimum $\boldsymbol{\omega}_t$ of the function

\begin{equation}\label{non_stationar_opt}
F_t(\boldsymbol{\omega}) = \mathbb{E}_{\mathcal{F}_{t-1}} f_{\xi_t}(\boldsymbol{\omega}, \mathbf{x}) \to \min_{\boldsymbol{\omega}}.
\end{equation}

More precisely, using the observations $L_1, L_2,\ldots,L_t$ and inputs $\mathbf{x}_i$ from training episodes $\xi_1, \xi_2,\ldots,\xi_t$ we need to construct an estimate $\widehat{\boldsymbol{\omega}}_t$ of an unknown vector $\boldsymbol{\omega}_t$ minimizing mean-risk functional~(\ref{non_stationar_opt}).



Consider the case where the data is such that the train tasks $t_i$ are homogeneous, and hence function~(\ref{multitask_loss}) belong to one distribution. For example, the Omniglot dataset satisfies this case. Then we construct the following SPSA based algorithm for finding parameters $\boldsymbol{\omega}_t$.

Let $\Delta_{ n } \in {\mathbb R}^d,\; n=1,2,\ldots$ be vectors consisting of independent random variables with Bernoulli distribution, called the {\it test randomized perturbation}, $\widehat{\boldsymbol{\omega}}_{0}$ is a vector with the initial values of weights, $\boldsymbol{\omega}^\star$ is a some point of minimum of functional~(\ref{non_stationar_opt}), $\{\alpha_n\}$ and $\{\beta_n\}$ are sequences of positive numbers. Then the SPSA few-shot learning algorithm builds the following estimates

\begin{eqnarray}\label{regression_opt}
\begin{cases}
L_n^{\pm} = L_n(\widehat{\boldsymbol{\omega}}_{n-1} \pm \beta_n \Delta_n)
\\
\\
\widehat{\boldsymbol{\omega}}_n = \widehat{\boldsymbol{\omega}}_{n-1} - \alpha_n \Delta_n \frac{L_n^{+} - L_n^{-}}{2 \beta_n}.
\end{cases}
\end{eqnarray}

{\it Assumption~1.} For $n=1,2,\ldots$, the successive differences $\Bar{\nu}_n=\nu_{2n}-\nu_{2n-1}$ of observation noise are bounded: $|\Bar{\nu}_n| \leq c_{\nu} < \infty$, or $\mathbb{E}\Bar{\nu}_n^2 \leq c_{\nu}^2$ if a sequence $\{\nu_t\}$ is random.

{\it Assumption~2.} Let assumptions 3.1--3.3 of the Theorem 3.1 from~\cite{granichin2015randomized} about strong convexity of $F_t$, Lipschitz condition of the gradient of $f_{\xi_t}$, local Lebesgue property and conditions for $\{\alpha_n\}$ and $\{\beta_n\}$ hold.    

For the considered additive external noise, we can suppose that this assumptions is satisfied due to the fact that this noise in~(\ref{fsl_loss}) is generated by support sets $S_{t_i}$ and query sets $Q_{t_i}$, and these sets are bounded for each task $t_i$. 

\begin{theorem}\label{thorem}
Let Asumptions 1, 2 and following conditions hold
\newline
(1) The learning sequence $\mathbf{x}_1, \mathbf{x}_2,\ldots, \mathbf{x}_n,\ldots$ consists of identically distributed independent random vectors;
\newline
(2) $\forall n\geq 1 $ the random vectors
$ \nu_1, \nu_2,  \ldots, \nu_n $ and
$\mathbf{x}_1, \mathbf{x}_2, \ldots, \mathbf{x}_{n-1}$
do not depend on
$ \mathbf{x}_n$ and $\Delta_{n}$,
and the random vector
$\mathbf{x}_n$  does not depend on
$\Delta_n$;
\newline
(3)  $\sum_n \alpha_n=\infty$ and
$\alpha_n\to 0,\;\beta_n \to 0,\;\alpha_n {\beta_n}^{-2} \to 0$ as
$n \to \infty$.

{\bf If} estimate sequence $\{\widehat{\boldsymbol{\omega}}_n\}$ generate by algorithm \eqref{regression_opt}
\newline
{\bf then} $\{\widehat{\boldsymbol{\omega}}_n\}$  converges in the mean-square sense: $
\lim_{n \to \infty}{ \mathbb{E}}\{\|\widehat{\boldsymbol{\omega}}_{n}-\boldsymbol{\omega}^{\star}\|^2\}= 0$.

Furthermore, {\bf if}
$
 \sum_n \alpha^n {\beta^n}^2 +{\alpha^n}^2 {\beta^n}^{-2}< \infty,
$
\newline
{\bf then}
$\widehat{\boldsymbol{\omega}}_n \to \boldsymbol{\omega}^{\star}$ as
$n\to \infty $
with probability $1$.
\end{theorem}

\begin{proof}
\begin{enumerate}
    \item Using the Assumption before this Theorem about bounding of an observation noise we can simplify conditions of the Theorem 3.1  from~\cite{granichin2015randomized} concerning $\{\alpha_n\}$ and $\{\beta_n\}$. As a result, we obtain the conditions of this Theorem.
    \item By the definition of~(\ref{nll}) and as for the sum of such functions in~(\ref{multitask_loss}), assumptions 3.1--3.3 of the Theorem 3.1 from~\cite{granichin2015randomized} are satisfied.
    \item The sequence $\{\Delta_n\}$ we use obviously satisfies the conditions of the Theorem 3.1 (see~\cite{granichin2015randomized}).
\end{enumerate}
\end{proof}

Now we can write our modified algorithm of the procedure of updating parameters $\theta$ of convolutional neural network for one training episode.

\begin{algorithm}\label{alg:spsa_protonet}
\begin{algorithmic}[1]
\caption{Training for episode $\xi_t: (t_1,\ldots,t_M)$}
\renewcommand{\algorithmicrequire}{\textbf{Input:}}
\renewcommand{\algorithmicensure}{\textbf{Output:}}
\REQUIRE $N_S$, $N_Q$, $N_C$, $\widehat{\boldsymbol{\omega}}_{t-1}$
\ENSURE Updated parameters $\theta$
\STATE Random sample $N_C$ classes
\FOR{$i$ in $\left\lbrace 1,\ldots,M \right\rbrace$}
\FOR{$k$ in $\left\lbrace 1,\ldots,N_C \right\rbrace$}
\STATE Random sample elements in $S_{t_i}^k$ 
\STATE Random sample elements in $Q_{t_i}^k$ 
\STATE Compute $\mathbf{c}^k_{t_i}$ via~(\ref{prototype}) 
\ENDFOR
\ENDFOR
\STATE $f_{\xi_t}=0$
\FOR{$i$ in $\left\lbrace 1,\ldots,M \right\rbrace$}
\STATE $\mathcal{L}_{\theta, t_i}=0$
\FOR{$k$ in $\left\lbrace 1,\ldots,N_C \right\rbrace$}
\FOR{($\mathbf{x},y)$ in $Q_{t_i}^k$ }
\STATE $\mathcal{L}_{\theta, t_i}=\mathcal{L}_{\theta, t_i} + \frac{1}{N_C N_Q} l_{\theta, t_i}^k (\mathbf{x})$
\ENDFOR
\ENDFOR
\STATE $f_{\xi_t}= f_{\xi_t} + \frac{1}{(\widehat{\omega}_t^i)^2} \mathcal{L}_{\theta, t_i} + \log (\widehat{\omega}_{t}^i)^2$
\ENDFOR
\STATE Update weights $\widehat{\boldsymbol{\omega}}_t$ via~(\ref{regression_opt})
\STATE Update parameters $\theta$ via SGD by $f_{\xi_t}$
\end{algorithmic}
\end{algorithm}

The inference of our approach during testing is identical to the inference of the prototypical networks approach.

\section{EXPERIMENTS}\label{sec:experiments}

We have experimented on the Omniglot dataset~\cite{lake2015human, lake2019omniglot} with the method for few-shot learning proposed in this paper. This dataset consists of 1623 handwritten characters (classes) collected from 50 alphabets. For each character there are 20 examples written by different people. For training and testing we used resized to $28 \times 28$ grayscale images. Examples from Omniglot are shown in Fig.~\ref{fig:tengwar_1_shot_20_way}. We used splitting of the dataset into 3 parts: for training, validation and testing. Alphabets and consequently classes in these parts do not intersect. Training part consists of 1028 classes from 33 alphabets, validation part consists of 172 classes from 5 alphabets, and testing part consists of 423 classes from 12 alphabets.

\begin{table*}[th]
	\caption{Omniglot, Original within alphabet}
	\label{table:table_omniglot_original_whithin}
	\begin{center}
		\begin{tabular}{p{5 cm}|c|c|c|c}
			\toprule
			Algorithm & 1-shot 20-way & 5-shot 20-way & 1-shot 5-way & 5-shot 5-way\\
			\midrule
			Prototypical Networks (ours realization) & 73.44 $\pm$ 0.6 \% & 87.2 $\pm$ 0.5 \% & 86.33 $\pm$ 0.6 \% & 94.94 $\pm$ 0.4 \% \\
			\midrule
			Multi-task pretrain, $M=3$, equal $\omega_t^i$ & 74 $\pm$ 0.61 \% & 86.84 $\pm$ 0.42 \% & 84.45 $\pm$ 0.7 \% & 94.06 $\pm$ 0.39 \% \\
			\hline
			Multi-task pretrain, $M=3$, random $\omega_t^i$ & 73.52 $\pm$ 0.6 \% & 86.91 $\pm$ 0.39 \% & 84.49 $\pm$ 0.7 \% & 94.32 $\pm$ 0.37 \% \\
			\midrule
			Multi-task pretrain, $M=3$, SPSA $\omega_t^i$ & {\bfseries 75.24 $\pm$ 0.59 \% } & {\bfseries 87.38 $\pm$ 0.4 \%} & {\bfseries 87.11 $\pm$ 0.7 \%} & {\bfseries 95.20 $\pm$ 0.4 \%} \\
			\hline
			Multi-task pretrain, $M=20$, SPSA $\omega_t^i$, Task sampling, $M_{top}=3$ & {\bfseries 74.33 $\pm$ 0.6 \%} & {\bfseries 87.48 $\pm$ 0.4 \%} & {\bfseries 86.67 $\pm$ 0.7 \%} & 94.93 $\pm$ 0.4 \% \\
			\hline
			Multi-task pretrain, $M=20$, SPSA $\omega_t^i$, Task sampling, $M_{top}=15$ & {\bfseries 74.20 $\pm$ 0.6 \%} & 87.34 $\pm$ 0.4 \% & {\bfseries 86.83 $\pm$ 0.7 \%} & {\bfseries 95.51 $\pm$ 0.4 \%} \\
			\hline
			Multi-task pretrain, $M=10$, SPSA $\omega_t^i$ & {\bfseries 73.91 $\pm$ 0.6 \% } & 87.24 $\pm$ 0.7 \% & {\bfseries 87.15 $\pm$ 0.7 \%} & {\bfseries 95.17 $\pm$ 0.4 \%} \\
			\hline
			Multi-task, $M=10$, SPSA $\omega_t^i$ & 69.95 $\pm$ 0.66 \%  & {\bfseries 88.14 $\pm$ 0.7 \%} & {\bfseries 88.8 $\pm$ 0.7 \%} & {\bfseries 96.4 $\pm$ 0.3 \%} \\
			\hline
			Multi-task, $M=15$, SPSA $\omega_t^i$ & 69.63 $\pm$ 0.66 \%  & {\bfseries 88.12 $\pm$ 0.7 \%} & {\bfseries 88.86 $\pm$ 0.7 \%} & {\bfseries 96.12 $\pm$ 0.3 \%} \\
			\bottomrule
		\end{tabular}
	\end{center}
\end{table*}

Most few-shot learning papers~\cite{vinyals2016matching, santoro2016meta, snell2017prototypical, finn2017model, jamal2019task} describing experiments on Omniglot use the version from~\cite{vinyals2016matching}, in which the character classes are augmented with rotations in multiples of 90 degrees. This gives 6492 classes form 50 alphabets. Hence the number of classes in training, validation, test splitting also increases 4 times.

In~\cite{lake2019omniglot} it is claimed that although this augmented version from~\cite{vinyals2016matching} contributed a lot to the development of few-shot learning methods, it does not solve the original problem posed in~\cite{lake2015human}. More precisely in~\cite{lake2015human} is considered the problem of classification by 1 example between 20 classes (1-shot, 20-way) within one alphabet (see Fig.~\ref{fig:tengwar_1_shot_20_way}). Statement ``within alphabet'' means that in the each task $t_i$ characters (classes) belong to the same alphabet. This type of experiment is called in~\cite{lake2019omniglot} {\it Omniglot, Original within alphabet}. Setting of this type is more difficult and standard few-shot learning algorithms, including prototypical networks, significantly drop their accuracy. Therefore, we focused on this Omniglot setting. Type described in~\cite{vinyals2016matching} is called {\it Omniglot, Augmented between alphabet}. Statement ``between alphabet'' means that in the each task $t_i$ characters (classes) may belong to different alphabets. We also tested our algorithm on this type of setting.

For our experiments we used the same deep convolutional neural network as in~\cite{snell2017prototypical}. This CNN is composed of four
convolutional blocks. Each block consists of a 64-filter $3 \times 3$ convolution, batch normalization layer, a ReLU (rectified linear unit) nonlinearity and a $2 \times 2$ max-pooling layer.

We have trained all our models during 50 epochs, where 1 epoch includes 100 random training episodes. For training parameters $\theta$ of the CNN was used SGD with Adam as in~\cite{snell2017prototypical}. Initial learning rate was $10^{-3}$, and the learning rate was cut in half every 2000 episodes. Parameters of the SPSA few-shot learning: $\gamma=1 / 6, \alpha^n=0.25 / n^{\gamma}, \beta^n=15 / n^{\frac{\gamma}{4}}$. These parameters were selected according to the theoretical results from~\cite{granichin2015randomized} and remained unchanged for all experiments.

We have experimented with several additional features for training our model. One of them is {\it pretraining}: first 20 epochs CNN was trained via vanilla prototypical networks (Algorithm 1), then 30 epochs it was trained by SPSA few-shot learning (Algorithm 2). 

Another feature is {\it task sampling}: 30 tasks are randomly selected, and then $M_{top}$ tasks are selected among them. The idea is to select the most different tasks for the training episode $\xi_t$. Each task $t$ is described by a set of prototypes $\{\mathbf{c}^k_t\}_{k=1}^{N_C}$. Then the differences between the two tasks $t_1$ and $t_2$ is calculated as $d(t_1, t_2)=\max_{k \in \{1,\ldots,N_C\}} \|\mathbf{c}^k_{t_1} - \mathbf{c}^k_{t_2}\|^2.$

We computed classification accuracy for our models averaged over 1000 randomly generated episodes from the test set and reported it with $95 \%$ confidence intervals. We also reported results for the original prototypical networks algorithm and for our method without SPSA weights optimization but with equal and random weights. These results are presented in Table~\ref{table:table_omniglot_original_whithin} and Table~\ref{table:table_omniglot_augment_between}.

\begin{table*}[h]
	\caption{Omniglot, Augmented between alphabet}
	\label{table:table_omniglot_augment_between}
	\begin{center}
		\begin{tabular}{p{5 cm}|c|c|c|c}
			\toprule
			Algorithm & 1-shot 20-way & 5-shot 20-way & 1-shot 5-way & 5-shot 5-way\\
			\midrule
			Prototypical Networks (ours realization) & 94.85 $\pm$ 0.18 \% & 98.62 $\pm$ 0.06 \% & 98.44 $\pm$ 0.17 \% & 99.56 $\pm$ 0.07 \% \\
			\midrule
			Multi-task pretrain, $M=3$, equal $\omega_t^i$ & 94.6 $\pm$ 0.17 \% & 98.54 $\pm$ 0.06 \% & 98.29 $\pm$ 0.19 \% & 99.57 $\pm$ 0.06 \% \\
			\hline
			Multi-task pretrain, $M=3$, random $\omega_t^i$ & 94.58 $\pm$ 0.17 \% & 98.53 $\pm$ 0.06 \% & 98.27 $\pm$ 0.19 \% & 99.56 $\pm$ 0.07 \% \\
			\midrule
			Multi-task pretrain, $M=3$, SPSA $\omega_t^i$ & {\bfseries 95.14 $\pm$ 0.18 \% } & {\bfseries 98.72 $\pm$ 0.06 \% } & {\bfseries 98.55 $\pm$ 0.17 \% } & 99.55 $\pm$ 0.07 \% \\
			\hline
			Multi-task pretrain, $M=20$, SPSA $\omega_t^i$, Task sampling, $M_{top}=3$ & {\bfseries 94.94 $\pm$ 0.16 \% } & {\bfseries 98.68 $\pm$ 0.06 \% } & {\bfseries 98.49 $\pm$ 0.18 \% } & {\bfseries 99.65 $\pm$ 0.07 \% } \\
			\hline
			Multi-task pretrain, $M=20$, SPSA $\omega_t^i$, Task sampling, $M_{top}=15$ & 94.45 $\pm$ 0.17 \% & 98.45 $\pm$ 0.06 \% & 98.27 $\pm$ 0.19 \% & {\bfseries 99.58 $\pm$ 0.07 \% } \\
			\hline
			Multi-task pretrain, $M=10$, SPSA $\omega_t^i$ & {\bfseries 95.24 $\pm$ 0.16 \% } & 98.43 $\pm$ 0.06 \% & 98.29 $\pm$ 0.19 \% & {\bfseries 99.56 $\pm$ 0.06 \% } \\
			\hline
			Multi-task, $M=10$, SPSA $\omega_t^i$ & 94.73 $\pm$ 0.16 \% & 98.46 $\pm$ 0.06 \% & 98.35 $\pm$ 0.19 \% & {\bfseries 99.56 $\pm$ 0.06 \% } \\
			\hline
			Multi-task, $M=15$, SPSA $\omega_t^i$ & {\bfseries 94.86 $\pm$ 0.17 \% } & 98.57 $\pm$ 0.07 \% & 98.4 $\pm$ 0.18 \% & {\bfseries 99.58 $\pm$ 0.05 \% } \\
			\bottomrule
		\end{tabular}
	\end{center}
\end{table*}

\begin{figure}[thpb]
\centering
\subfloat[Omniglot, Original within alphabet.]{%
  \includegraphics[clip,width=0.62\columnwidth]{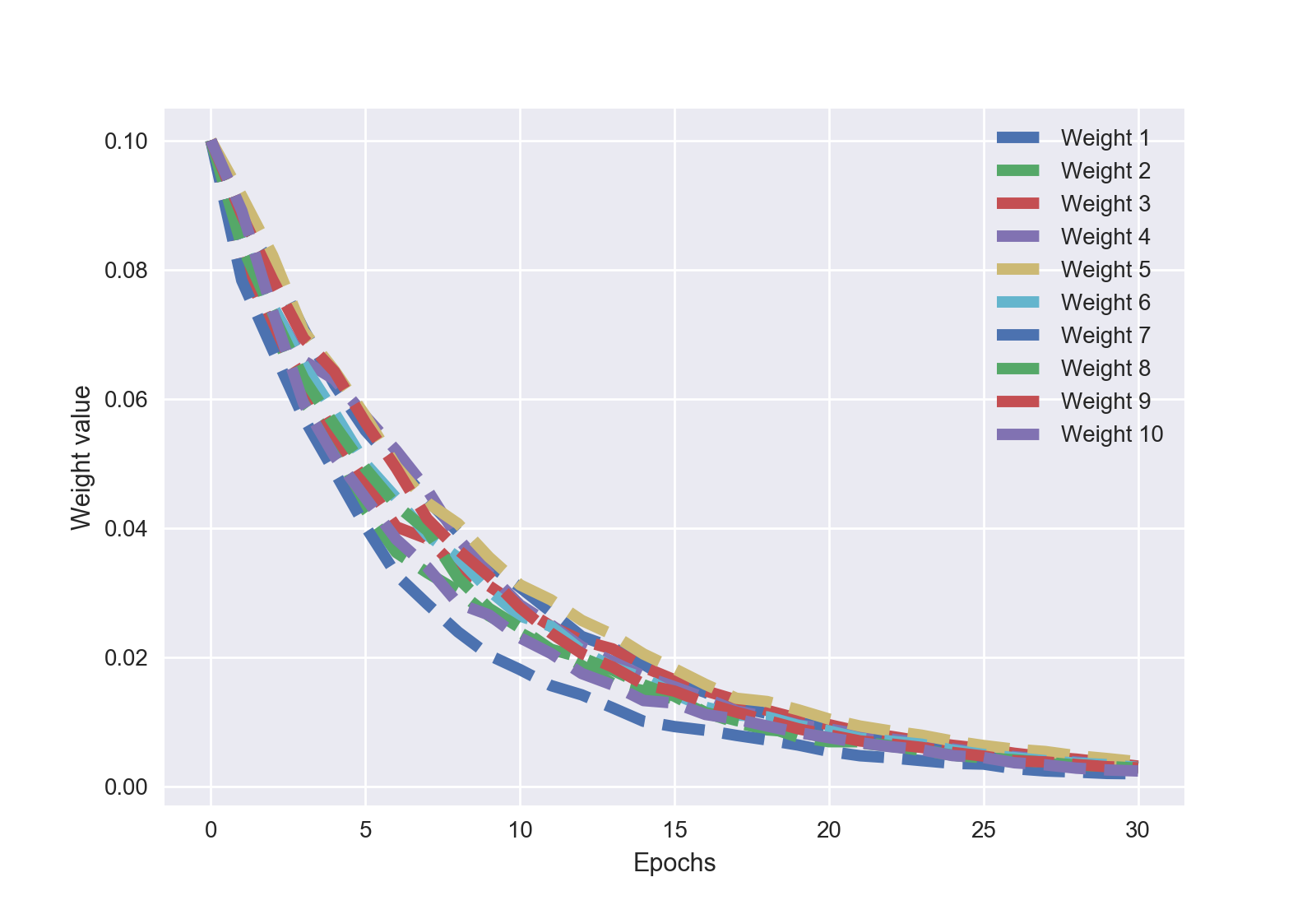}%
}

\subfloat[Omniglot, Augmented between alphabet.]{%
  \includegraphics[clip,width=0.62\columnwidth]{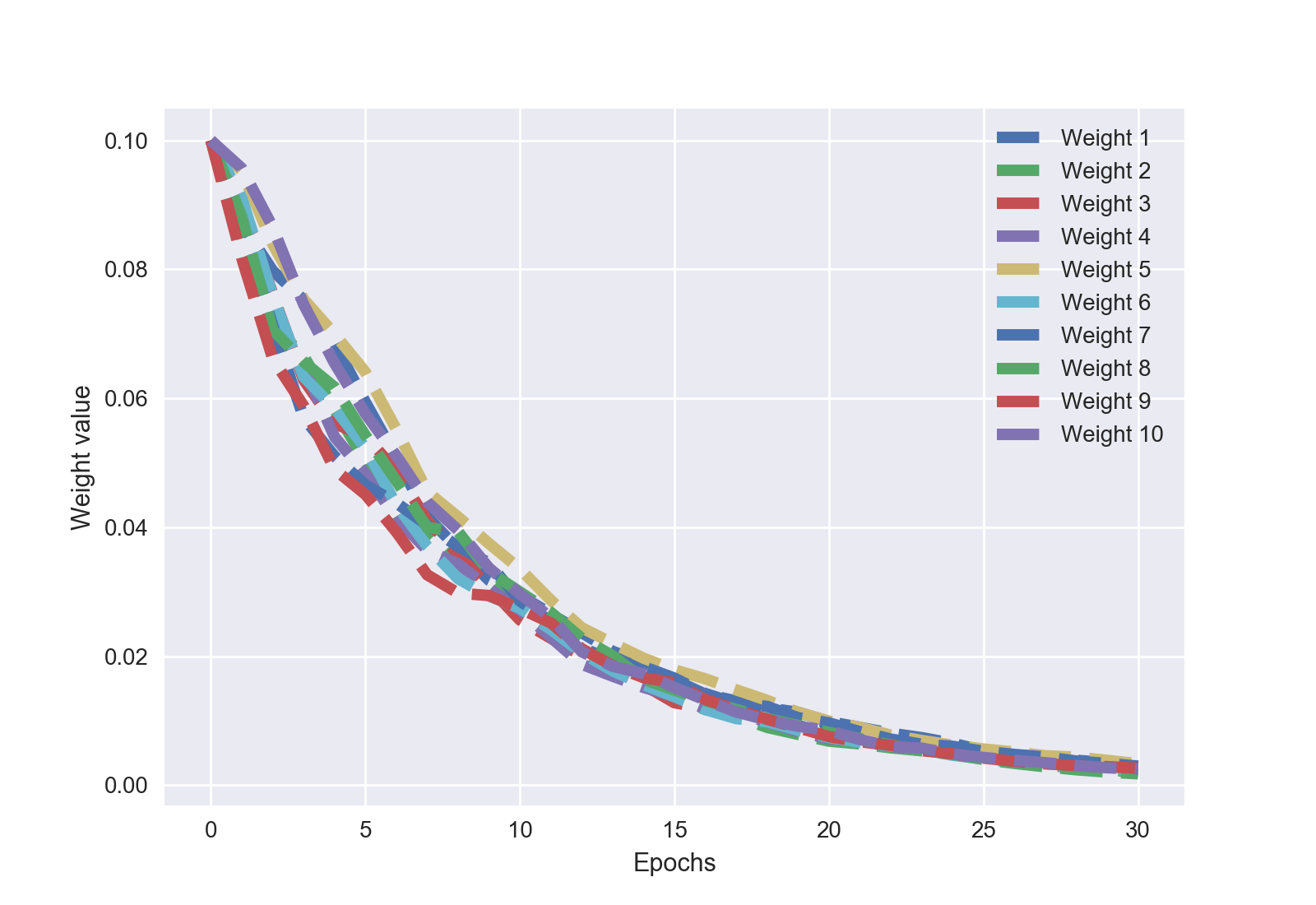}%
}

\caption{SPSA $w_t^i, i=1,\ldots,10$ values by learning epochs, 1-shot 20-way.}
\label{fig:weights_1_20}
\end{figure}




As can be seen from the results of experiments, proposed in this paper approach significantly outperform prototypical networks in the original within alphabet Omniglot setting, on which our attention was focused. Method with pretraining and $M=3$ demonstrates best average result and best result for the 20-way few-shot classification. Methods without pretraining and with larger $M=10$ and $M=15$ give best result for the 5-way classification problem.

In the augmented between alphabet setting our method also surpass or not inferior to the original algorithm. It is important to note that experiments with multi-task but without SPSA weights optimization demonstrate significantly worse accuracy. This fact illustrates an importance of the proposed SPSA based algorithm. 

Consider the behaviour during training of the weights $\boldsymbol{\omega}_t$ with $M=10$ (Multi-task, $M=10$, SPSA $\omega_t^i$ in Tables~\ref{table:table_omniglot_original_whithin} and~\ref{table:table_omniglot_augment_between}). Fig.~\ref{fig:weights_1_20} shows the values of weights depending on the training epoch for the 1-shot 20-way classification problem for both Omniglot settings. As can bee seen from this Figure, the weights gradually converge to small values, which indicates an increase in the contribution of tasks to the last epochs of training. 


\section{CONCLUSIONS}\label{sec:conclusion}

In this paper we described and gave a theoretical justification of the SPSA-like multi-task modification of the prototypical networks algorithm for few-shot learning. The proposed approach outperforms original method in several settings of the standard Omniglot tests. In future works, we plan to extend and combine this approach with other main few-shot learning algorithms and also try to use it to study clusters in graphs, in particular, for modeling social processes and phenomena. In addition, we plan to combine the described method with a promising projective optimization approach~\cite{senov2017accelerating}.

\addtolength{\textheight}{-12cm}   





\bibliographystyle{IEEEtran}
\bibliography{IEEEabrv,biblio}

\end{document}